\documentclass{article}

\usepackage{times}
\usepackage{fullpage}
\usepackage{graphicx} 
\usepackage{pgf}
\usepackage{subfigure}
\usepackage{enumitem}

\usepackage{natbib}

\usepackage{algorithm}
\usepackage{algorithmic}

\newif\ifcolorfuldiscussion
\colorfuldiscussiontrue

\newif\ifstorytelling
\storytellingfalse

\usepackage{verbatim}
\usepackage{blindtext}
\usepackage{amsthm}
\usepackage{amsmath}
\usepackage{amssymb}
\usepackage{algorithm}
\usepackage{algorithmic}

\newtheorem{theorem}{Theorem}[section]
\newtheorem{claim}{Claim}[section]
\newtheorem{definition}{Definition}[section]
\newtheorem{lemma}[theorem]{Lemma}
\newtheorem{proposition}[theorem]{Proposition}

\newcommand{\rdt}{properly filled facets}
\newcommand{\R}{\mathbb{R}}
\usepackage{hyperref}

\begin{document} 

\title{Intersecting Faces: \\Non-negative Matrix Factorization With New Guarantees}

\author{Rong Ge\\Microsoft Research New England\\rongge@microsoft.com
 \and James Zou\\
Microsoft Research New England\\jazo@microsoft.com}
\date{}
\maketitle

\begin{abstract}
Non-negative matrix factorization (NMF) is a natural model of admixture and is widely used in science and engineering. A plethora of algorithms have been developed to tackle NMF, but due to the non-convex nature of the problem, there is little guarantee on how well these methods work. Recently a surge of research have focused on a very restricted class of NMFs, called separable NMF, where provably correct algorithms have been developed. In this paper, we propose the notion of subset-separable NMF, which substantially generalizes the property of separability. We show that subset-separability is a natural necessary condition for the factorization to be unique or to have minimum volume. We developed the Face-Intersect algorithm which provably and efficiently solves subset-separable NMF under natural conditions, and we prove that our algorithm is robust to small noise. We explored the performance of Face-Intersect on simulations and discuss settings where it empirically outperformed the state-of-art methods. Our work is a step towards finding provably correct algorithms that solve large classes of NMF problems. 
\end{abstract}

\section{Introduction}

In many settings in science and engineering the observed data are admixtures of multiple latent sources. We would typically want to infer the latent sources as well as the admixture distribution given the observations.  Non-negative matrix factorization (NMF) is a natural mathematical framework to model many admixture problems. 

In NMF we are given an observation matrix $M\in \R^{n\times m}$, where each row of $M$ corresponds to a data-point in $\R^m$. We assume that there are $r$ latent sources, modeled by the unobserved matrix $W \in \R^{r\times m}$, where each row of $M$ characterizes one source. Each observed data-point is a linear combination of the $r$ sources and the combination weights are encoded in a matrix $A \in \R^{n\times r}$.  Moreover, in many natural settings, the sources are non-negative and the combinations are additive. The computational problem is then is to factor a given matrix $M$ as $M = AW$, where all the entries of $M, A$ and $W$ are non-negative. We call $r$ the inner-dimension of the factorization, and the smallest possible $r$ is usually called the nonnegative rank of $M$. NMF was first purposed by \citep{LS99}, and has been widely applied in computer vision~\citep{LS00}, document clustering~\citep{XLG}, hyperspectral unmixing\citep{VCA,NFINDR}, computational biology \citep{NMF_bio}, etc. We give two concrete examples

\textbf{Example 1.} In topic modeling, $M$ is the $n$-by-$m$ word-by-document matrix, where $n$ is the vocabulary size and $m$ is the number of documents. Each column of $M$ corresponds to one document and the entry $M(i,j)$ is the frequency with which word $i$ appears in document $j$. The topics are the columns of $A$, and $A(i,k)$ is the probability that topic $k$ uses word $i$. $W$ is the topic-by-document matrix and captures how much each topic contributes to each document. Since all the entries of $M, A$ and $W$ are frequencies, they are all non-negative. Given $M$ from a corpus of documents, we would like to factor $M = AW$ and recover the relevant topics in these documents.  (Note that in this example $A$ is the matrix of ``sources'' and $W$ is the matrix of mixing weights, so it is the transpose of what we just introduced. We use this notation to be consistent with previous works \citep{AGKM}.)

\textbf{Example 2.} In many bio-medical applications, we collect samples and for each sample perform multiple measurements (e.g. expression of $10^4$ genes or DNA methylation at $10^6$ positions in the genome; all the values are non-negative). $M$ is the sample-by-measurement matrix, where $M(i,j)$ is the value of the $j$th measurement in sample $i$. Each sample, whether taken from humans or animals, is typically a composition of several cell-types that we do not directly observe. Each row of $W$ corresponds to one cell-type, and $W(k,j)$ is the value of cell-type $k$ in measurement $j$. The entry $A(i,k)$ is the fraction of sample $i$ that consists of cell-type $k$. Experiments give us the matrix $M$, and we would like to factor $M = AW$ to identify the relevant cell-types and their compositions in our samples.

Despite the simplicity of its formulation, NMF is a challenging problem. First, the NMF problem may not be identifiable, and hence we can not hope to recover the \textit{true} $A$ and $W$. Moreover, even ignoring the identifiability
 \citet{Vav} showed that finding any factorization $M = AW$ with inner-dimension $r$ is an $NP$-hard problem. \citet{AGKM} showed under reasonable assumptions we cannot hope to find a factorization in time $(mn)^{o(r)}$, and the best algorithm known is \citet{DBLP:conf/soda/Moitra13} that runs in time $O(2^r mn)^{O(r^2)}$.

Many heuristic algorithms have been developed for NMF but they do not have guarantees for when they would actually converge to the true factorization \citep{LS00, lin2007projected}. More recently, there has been a surge of interest in constructing practical NMF algorithms with strong theoretical guarantees. Most of this activity (e.g. \citet{AGKM,BRRT, Kumar12,G,GillisVavasis2}, see more in \citet{NMFsurvey}) are based on the notion of \textit{separability}\cite{DS} which is a very strict condition that requires that all the rows of $W$ appear as rows in $M$. While this might hold in some document corpus, it is unlikely to be true in other engineering and bio-medical applications.

\paragraph{Our Results} In this paper, we develop the notion of \textit{subset separability}, which is a significantly weaker and more general condition than separability. In topic models, for example, separability states that there is a word that is unique to each topic. Subset separability means that there is a combination of words that is unique to each topic. We show that subset separability arise naturally as a necessary condition when the NMF is identifiable or when we are seeking the minimal volume factorization. We characterize settings when subset-separable NMF can be solved in polynomial-time, and this include the separable setting as a special case. We construct the Face-Intersect algorithm which provably and robustly solves the NMF even in the presence of adversarial noise. We use simulations to explore conditions where our algorithm achieves more accurate inference than current state-of-art algorithms.

\paragraph{Organization}
We first describe the geometric interpretation of NMF (Sec. 2), which leads us to the notion of subset-separable NMF (Sec. 3). We then develop our Face-Intersect algorithm and analyze its robustness (Sec. 4). Our main result, Theorem \ref{thm:mainrobust}, states that for subset-separable NMF, if the facets are \textit{properly filled} in a way that depends on the magnitude of the adversarial noise, then Face-Intersect is guaranteed to find a factorization that is close to the true factorization in polynomial time. We discuss the algorithm in more detail in Sections 5 and 6, and analyze a generative model that give rise to properly filled facets in Section 7. Finally we present experiments to explore settings where Face-Intersect outperforms state-of-art NMF algorithms (Sec. 8). Due to space constraints, all the proofs are presented in the appendix. Throughout the paper, we give intuitions behind proofs of the main results.   

\section{Geometric intuition}

For a matrix $M \in \R^{n\times m}$, we use $M^i \in \R^m$ to denote the $i$-th row of $M$, but it is viewed as a column vector. Given a factorization $M = AW$, without loss of generality we can assume the rows of $M,A,W$ all sum up to 1 (this can always be done, see \cite{AGKM}). In this way we can view the rows of $W$ as vertices of an unknown simplex, and the rows of $M$ are all in the convex hull of these vertices. The NMF is then equivalent to the following geometric problem: 

\paragraph{NMF, Geometric Interpretation} There is an unknown $W$-simplex whose vertices are the rows of $W \in \mathbb{R}^{m}$, $W^1, ..., W^r$. We observe $n$ points $M^1,M^2,...,M^n \in \R^m$ (corresponding to rows of $M$) that lie in the $W$-simplex. The goal is to identify the vertices of the $W$-simplex.

When clear from context, we also call the $W$ matrix as the simplex, and the goal is to find the vertices of this simplex. There is one setting where it is easy to identify all the vertices.  

\begin{definition}[separability]
A NMF is separable if all the vertices $W^j$'s appear in the points $M^i$'s that we observe. 
\end{definition}

Separability was introduced in \citet{DS}. When the NMF is separable, the problem simplifies as we only need to identify which of the points $M^j$'s are vertices of the simplex. This can be done in time polynomial in $n, m$ and $r$ \citep{AGKM}. Separability is a highly restrictive condition and it takes advantage of only the $0$-dimensional structure (vertices) of the simplex. In this work, we use higher dimensional structures of the simplex to solve the NMF. We use the following standard definition of facets:

\begin{definition}[facet]
A facet $S \subset [r]$  of the $W$-simplex is the convex hull of vertices $\{W^j: j \in S\}$. We call $S$ a filled facet if there is at least one point $M^i$ in the interior of $S$ (or if $|S| = 1$ and there is one point $M^i$ that is equal to that vertex; such $M^i$ is called an anchor). 
\end{definition}

\paragraph{Conventions}  When it's clear from context, we interchangeably represent a facet $S$ both by the indices of its vertices and by the convex hull of these vertices. A facet also corresponds to a unique linear subspace $Q_S$ with dimension $|S|$ that is the span of $\{W^j: j\in S\}$. In the rest of the paper, it's convenient to use linear algebra to quantify various geometric ideas. We will represent a $d$-dimensional subspace of $\R^m$ using a matrix $U\in \R^{m\times d}$, the columns of matrix $U$ is an arbitrary orthonormal basis for the subspace (hence the representation is not unique). We use $P_{U} = UU^T$ to denote the projection matrix to subspace $U$, and $U^\perp\in \R^{m\times (m-d)}$ to denote an arbitrary representation of the orthogonal subspace. For two subspaces $U$ and $V$ of the same dimension, we define their distance to the the $\sin$ of the principle angle between the two subspaces (this is the largest angle between vectors $u,v$ for $u\in U$ and $v\in V$). This distance can be computed as the spectral norm $\|P_{U^\perp}V\|$ (and has many equivalent formulations). 

\section{Subset Separability}

NMF is not identifiable up to scalings and permutations of the rows of $W$. 
Ignoring such transformations, there can still be multiple non-negative factorizations of the same matrix $M$. This arise when there are different sets of $r$ vertices in the non-negative orthant that contain all the points $M^i$ in its convex hull. For example, suppose $M = AW$ and the $A$ matrix has all positive entries. All the points $M^i$ are in the interior of the $W$-simplex. Then it is possible to perturb the vertices of $W$ while still maintaining all of the $M^i$'s in its convex hull. This give rise to a different factorization $M = \hat{A}\hat{W}$. When the factorization is not unique, we may want find a solution where the $W$-simplex has minimal volume, in the sense that it is impossible to move a single vertex and shrink the volume while maintaining the validity of the solution.

It's clear that in order for $W$ to be the minimal volume solution to the NMF, there must be some points $M^i$ that lie on the boundary of the $W$-simplex. We show that a necessary condition for $W$ to be volume minimizing is for the filled facets (facets of $W$ with points in its interior) to be subset-separable. Intuitively, this means that each vertex of $W$ is the unique intersection point of a subset of filled facets. 
 
\begin{definition}[subset-separable]
A NMF $M = AW$ is subset-separable if there is a set of filled facets $S_1, ..., S_k \subset [r]$ such that $\forall j \in [r]$, there is a subset of $S_{j_1}, S_{j_2}, ..., S_{j_{k_j}}$ whose intersection is exactly $j$. 
\end{definition} 

\begin{proposition}
\label{prop:volume}
Suppose $W$ is a minimal volume rank $r$ solution of the NMF $M = AW$. Then $W$ is subset-separable. 
\end{proposition}

It is easy to see that the factorization $M = AW$ is subset-separable is equivalent to the property that for every $j_1 \neq j_2 \in [r]$, there is a row $i$ of $A$ such that $A_{i, j_1} = 0$ and $A_{i, j_2} \neq 0$. The previously proposed separability condition corresponds to the special case where the filled facets $S_1, ..., S_k$ correspond to the singleton sets $\{W^1\}, ..., \{W^r\}$. 

\begin{figure}
\centering
\includegraphics[trim=0cm 7cm 0 0cm, width=0.5\textwidth]{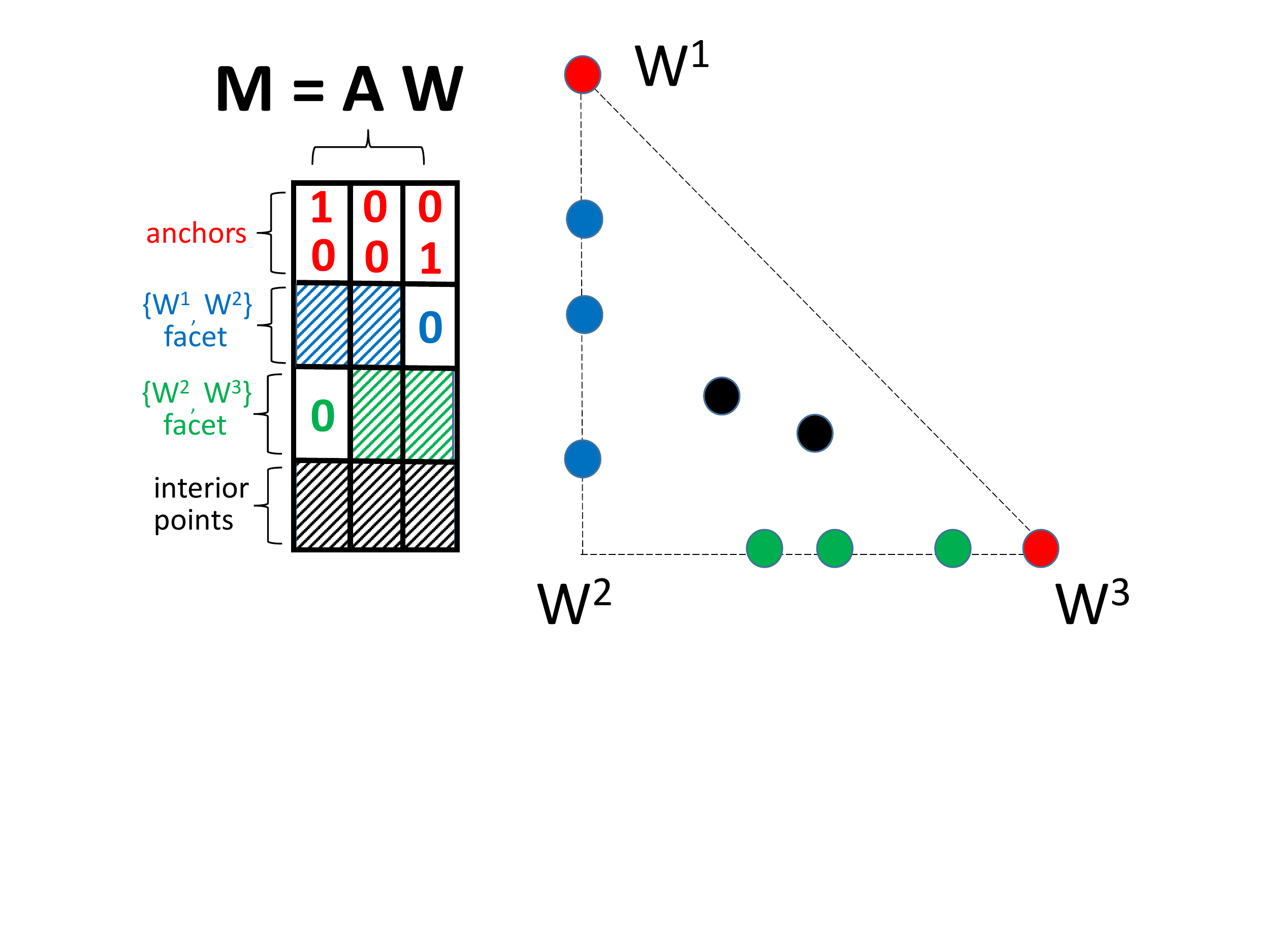}
\caption{Illustration of the NMF geometry.}
\label{fig:fig1}
\end{figure}

\textbf{Example.} We illustrate the subset-separable condition in Figure~\ref{fig:fig1}. In this figure, the circles correspond to data points $M^i$'s and they are colored according to the facet that they belong. The filled facets are $S_1 = \{1\}, S_2 = \{ 3\}, S_3 = \{1,2\}$ and $S_4 = \{2,3\}$. The facet $\{W^1, W^3\}$ is not filled since there are no points in its interior. The singleton facets $S_1$ and $S_2$ are also called anchors. This NMF is subset-separable since $W^2$ is the unique intersection of $S_3$ and $S_4$, but it is not separable. The figure also illustrates the corresponding $A$ matrix, where the rows are grouped by facets and the shaded entries denote the support of each row.

The geometry of the simplex suggests an intuitive meta-algorithm for solving subset-separable NMFs, which is the basis of our Face-Intersect algorithm. 
\begin{enumerate}\itemsep=0pt
\item Identify the filled facets, $S_1, ..., S_k$, $r \le k \le n$.
\item Take intersections of the facets to recover all the rows of $W$ (vertices of the simplex).
\item Use $M$ and $W$ to solve for $A$.  
\end{enumerate}

\section{Robust algorithm for subset-separable NMF}
In order to carry out the meta-algorithm, the key computational challenge is to efficiently and correctly identify the filled facets of the $W$ simplex. 
Finding filled facets is related to well-studied problems in subspace clustering \citep{vidal2010tutorial} and subspace recovery\citep{HardtM13}. In subspace clustering we are given points in $k$ different subspaces and the goal is to cluster the points according to which subspace it belong to. This problem is in general NP-hard \citep{elhamifar2009sparse} and can only be solved under strong assumptions. Subspace recovery tries to find a unique subspace a fraction $p$ of the points. \citet{HardtM13} showed this problem is hard unless $p$ is large compared to the ratio of the dimensions. Techniques and algorithms from subspace clustering and recovery typically make strong assumptions about the independence of subspaces or the generative model of the points, and cannot be directly applied to our problem. Moreover, our filled facets have the useful property that they are on the boundary of the convex hull of the data points, which is not considered in general subspace clustering/discovery methods. 
We identified a general class of filled facets, called \textit{properly filled facets} that are computationally efficient to find. 

\begin{definition}[properly filled facets]
Given a NMF $M = AW$, a set of facets $S_1, ..., S_k \in [r]$ of $W$ is properly filled if it satisfies the following properties:
\begin{enumerate}
\item For any facet $|S_i| > 1$, the rows of $A$ with support equal to $S_i$ (i.e. points that lie on this facet) has a $|S_i|-1$-dimensional convex hull. Moreover, there is at least one row of $A$ that is in the interior of the convex hull. 
\item (General positions property.) For any subspace of dimension $1 < t < r$, if it contains more than $t$ rows in $M$, then the subspace contains at least one $S_i$ which is not a singleton facet. 
\end{enumerate}
\end{definition}

Condition 1 ensures that each $S_i$ has sufficiently many points to be non-degenerate. Condition 2 says that points that are not in the lower dimensional facets $S_1, ..., S_k$ are in general positions, so that no random subspace look like a properly filled facet. A set of properly filled facets $S_1, ..., S_k$ may contain singleton sets corresponding some of the rows $W^j$ if these rows also appear as rows in $M$.  
We first state the main results and then state the Face-Intersect algorithm.

\begin{theorem}
Suppose $M = AW$ is subset separable by $S_1, ..., S_k$ and these facets are properly filled, then given $M$ the Face-Intersect algorithm computes $A$ and $W$ in time polynomial in $n$, $m$ and $r$ (and in particular the factorization is unique).
\end{theorem}

In many applications, we have to deal with noisy NMF $\hat{M} = AW + \mbox{noise}$ where (potentially correlated) noise is added to rows of the data matrix $M$. Suppose every row is perturbed by a small noise $\epsilon$ (in $\ell_2$ norm), we would like the algorithm to be robust to such additive noise. We need a generalization of properly filled facets. 

\begin{definition}[$(N,H,\gamma)$ properly filled facets]
\label{def:robust}
 Given a NMF $M = AW$, a set of facets $S_1, ..., S_k \in [r]$ of $W$ is $(N,H,\gamma)$ properly filled if it satisfies the following properties:
\begin{enumerate}
\item In any set $|S_i| > 1$, there is a row $i^*$ in $A$ whose support is equal to $S_i$, and is in the convex hull of other rows of $A$. There exists a convex combination $M^{i^*} = \sum_{i \in [n]\backslash i^*} w_i M^i$, such that the matrix $\sum_{i \in [n]\backslash i^*} w_i (M^i) (M^i)^T$ has rank $|S_i|$, and the smallest nonzero singular value is at least $\gamma$. We call this special point $M^{i^*}$ the {\em center} for this facet.
\item For any set $|S_i| > 1$, there are at least $N$ rows in $A$ whose support is exactly equal to $S_i$.
\item For any subspace $Q$ of dimension $1 < t < r$, if there are at least $N$ rows of $M$ in an $\epsilon$-neighborhood of $Q$, then there exists a non-singleton set $S_i$ with corresponding subspace $Q_i$ such that $\|P_{Q^\perp} Q_i\| \le H\epsilon$.
\end{enumerate}
\end{definition}

Intuitively, if we represent the center point as a convex combination of other points, the only points that have a nonzero contribution must be on the same facet as the center. Condition 1 then ensures there is a ``nice'' convex combination that allows us to robustly recover the subspace corresponding to the facet even in presence of noise. Condition 2 shows every properly filled facets contain many points, which is why they are different from other subspaces and are the facets of the true solution. Condition 3 is a generalization of the general position propery, which essentially says ``every subspace that contains many points must be close to a properly filled facet''. In Section~\ref{sec:genmain} we show that under a natural generative model, the NMF has $(N,H,\gamma)$-properly filled facets with high probability.

Properly filled facets is a property of how the points $M^i$ are distributed on the facets of $W$. The geometry of the $W$-simplex itself also affects the accuracy of our Face-Intersect algorithm.

\begin{definition}
A matrix $W \in \mathbb{R}^{r\times m} (r \le m)$  is $\alpha$-robust if its rows have norm bounded by $1$, and its $r$-th singular value is at least $\alpha$.
\end{definition}

Under these assumptions we prove that Face-Intersect robustly learns the unknown simplex $W$.

\begin{theorem}
\label{thm:mainrobust}
Suppose $M = AW$ is subset separable by $S_1, ..., S_k$ and these facets are $(N,H,\gamma)$ properly filled, and the matrix $W$ is $\alpha$-robust. Then given $\hat{M}$ whose rows are within $\ell_2$ distance $\epsilon$ to $M$, with $\epsilon < o(\alpha^4\gamma/Hr^3)$, Algorithm Face-Intersect finds $\hat{W}$ such that there exists a permutation $\pi$ and for all $i$ $\|\hat{W}_i - W_{\pi(i)}\| \le O(H r^2\epsilon/\alpha^2 \gamma)$. The running time is polynomial in $n, m$ and $r$. 
\end{theorem}

\begin{algorithm}
\begin{algorithmic}
\STATE Run Algorithm~\ref{alg:findall} to find subspaces that correspond to properly filled facets $S_1,S_2,...,S_k$ where $|S_i|\ge 2$.
\STATE Run Algorithm 5 to find the \textit{intersection vertices} $P$.
\STATE Run Algorithm~\ref{alg:findremaining} (similar to Algorithm 4 in \cite{AroraEtAl_icml13}) to find the singleton points (anchors).
\STATE Given $\hat{M}$, $\hat{W}$, compute $\hat{A}$.
\end{algorithmic}
\caption{Face-Intersect}\label{alg:main}
\end{algorithm}

A vertex $j \in [r]$ is an intersection vertex if there exists a subset of properly filled facets $\{S_{j_k}: |S_{j_k}| \geq 2\}$ such that $j = \cap_{k} S_{j_k}$. Since the first module of Face-Intersect, Algorithm~\ref{alg:findall}, only finds non-singleton facets, the intersection vertices are all the vertices that we could find using these facets. The last module of Face-Intersect finds all the remaining vertices of the simplex.

\paragraph{Our approach} The main idea of our algorithm is to first find the subspaces corresponding properly filled facets, then take the intersections of these facets to find the intersection vertices. Finally we adapt the algorithm from \cite{AroraEtAl_icml13} to find the remaining vertices that correspond to singleton sets.

\begin{itemize}[noitemsep,nolistsep]
\item {\bf Finding facets} For each row of $M$, we try to represent it as the convex combination of other rows of $M$. We use an iterative algorithm to make sure the span of points used in this convex combination is exactly the subspace corresponding to the facet. 
\item {\bf Removing false positives} The previous step will generate subspaces that correspond to properly filled facets, but it might also generate false positives (subspaces that do not correspond to any properly filled facets). Condition 3 in Definition~\ref{def:robust} allows us to filter out these false positives as these subspaces will not contain enough nearby points.
\item {\bf Finding intersection vertices} We design an algorithm that systematically tries to take the intersections of subspaces in order to find the intersection vertices. This relies on the subset-separable property and robustness properties of the simplex. This step computes at most $O(nr)$ subspace intersection operations.
\item {\bf Finding remaining vertices} The remaining vertices correspond to the singleton sets. This is similar to the separable case and we use an algorithm from \citet{AroraEtAl_icml13}.
\end{itemize}

\section{Finding properly filled facets}

In this section we show how to find properly filled facets $S_i$ with $|S_i| \geq 2$. The singleton facets (anchors) are not considered in this section, since they will be found through a separate algorithm. We first show how to find a properly filled facet if we know its center (Condition 1 in Definition~\ref{def:robust}). Then to find all the properly filled facets we enumerate points to be the center and remove  false positives.

\paragraph{Finding one properly filled facet}

Given the center point, if there is no noise then when we represent this point as convex combinations of other points, all the points with positive weight will be on the same facet. Intuitively the span of these points should be equal to the subspace corresponding to the facet. However there are two key challenges here: first we need to show that when there is noise, points with large weights in the convex combination are close to the true facet; second, it is possible that points with large weights only span a lower dimensional subspace of the facet. 
Condition 1 in Definition~\ref{def:robust} guarantees that there exists a \textit{nice} convex combination that spans the entire subspace (and robustly so because the smallest singular value is large compared to noise). In Algorithm~\ref{alg:findone}, we iteratively improve our convex combination and eventually converge to this nice combination. 

\begin{algorithm}
\begin{algorithmic}[1]
\INPUT points $\hat{v}^1, \hat{v}^2,...,\hat{v}^n$, and center point $\hat{v}^0$ (Condition 1 in Definition~\ref{def:robust}).
\OUTPUT the proper facet containing $\hat{v}^0$.
\STATE Maintain a subspace $\hat{Q}$ (initially empty)
\STATE Iteratively solve the following optimization program:
\begin{align*}
\max\quad  \mbox{tr}(P_{\hat{Q}^\perp}\sum_{i=1}^n w_i \hat{v}^i&(\hat{v}^i)^T P_{\hat{Q}^\perp}) \\
\forall i\in[n] \quad w_i &\ge 0\\
\sum_{i=1}^n w_i & = 1\\
\|\hat{v}^0 - \sum_{i=1}^n w_i \hat{v}^i\| &\le 2\epsilon\\
\mbox{diag}(\hat{Q}^T \left(\sum_{i=1}^n w_i \hat{v}^i(\hat{v}^i)^T\right) \hat{Q}) & \ge \gamma/2.
\end{align*}
\STATE Let $\hat{Q}$ be the top singular space of $\left(\sum_{i=1}^n w_i \hat{v}^i(\hat{v}^i)^T\right)$ for singular values larger than $\gamma/2d$.
\STATE Repeat until the dimension of $\hat{Q}$ does not increase.
\end{algorithmic}
\caption{Finding a properly filled facet}\label{alg:findone}
\end{algorithm}

\begin{theorem}
\label{thm:findonefacet}
Suppose $\|\hat{v}^i - v^i\|\le \epsilon$, $v^0$ is the center point of a properly filled facet $S\subset [r]$ with $|S| = d$, and the unknown simplex $W$ is $\alpha$-robust, when $d\sqrt{r}\epsilon/\alpha\gamma \ll 1$ Algorithm~\ref{alg:findone} stops within $d$ iterations, and the subspace $\hat{Q}$ is within distance $O(\sqrt{r}\epsilon/\alpha \gamma)$ to the true subspace $Q_S$.
\end{theorem}

The intuition of Algorithm~\ref{alg:findone} is to maintain a convex combination for the center point. We show for any convex combination, the top singular space associated with the combination, $\hat{Q}$, is always close to a subspace of the true space $Q_S$. The algorithm then tries to explore other directions by maximizing the projection that is outside the current subspace $\hat{Q}$ (the objective function of the convex optimization), while maintaining that the current subspace have large singular values (the last constraint). In the proof we show since there is a nice solution, the algorithm will always be able to make progress until the final solution is a \textit{nice} convex combination.

\paragraph{Finding all subsets}

Algorithm~\ref{alg:findone} can find one properly filled facet, if we have its center point (Condition 1 in Definition~\ref{def:robust}). In order to find all the properly filled facets, we enumerate through rows of $M$ and prune false positives using Condition 3 in Definition~\ref{def:robust}

\begin{algorithm}
\begin{algorithmic}[1]
\INPUT $\hat{M}$ whose factorization is subset-separable with $(N,H,\gamma)$-\rdt.
\FOR{$i = 1$ to $n$}
\STATE Let $\hat{v}^0 = \hat{M}^i$ and $\hat{v}^1,...,\hat{v}^{n-1}$ be the rest of vertices.
\STATE Run Algorithm~\ref{alg:findone} to get a subspace $Q$.
\STATE If $\mbox{dim}(Q) < r$, and there are at least $N$ points that are within distance $O(\sqrt{r}\epsilon/\alpha \gamma)$ add it to the collection of subspaces.
\ENDFOR
\STATE Let $Q$ be a subspace in the collection, remove $Q$ if there is a subspace $Q'$ with $\mbox{dim}(Q') < \mbox{dim}(Q)$ and  $\|P_{Q^\perp} Q'\| \le O(H\sqrt{r}\epsilon /\alpha \gamma)$
\STATE Merge all subspaces that are within distance $O(H\sqrt{r}\epsilon /\alpha \gamma)$ to each other.
\end{algorithmic}
\caption{Finding all proper facets}\label{alg:findall}
\end{algorithm}

\begin{theorem}
\label{thm:findall}
If $H\sqrt{r}\epsilon /\alpha \gamma = o(\alpha)$, then the output of Algorithm~\ref{alg:findall} contains only subspaces that are $\epsilon_S = O(H\sqrt{r}\epsilon /\alpha \gamma)$-close to the properly filled facets, and for every properly filled facet there is a subspace in the output that is $\epsilon_S$ close.
\end{theorem}

\section{Finding intersections}

Given an subset-separable NMF with $(N,H,\gamma)$-\rdt, let $Q_i$ denote the subspace associated with a set $S_i$ of vertices: $Q_i = \mbox{span}(W^{S_i})$. For all properly filled facets with at least two vertices, Algorithm~\ref{alg:findall} returns noisy versions of the subspaces $\hat{Q}_i$ that are $\epsilon_S$ close to the true subspaces. Without loss of generality, assume the first $h$ facets are non-singletons. Our goal is to find all the intersection vertices $\{W^i:i\in P\}$. Recall that intersection vertices are the unique intersections of subsets of $S_1, ..., S_h$. We can view this as a set intersection problem:

\paragraph{Set Intersections} We are given sets $S_1,S_2,...,S_h \subset [r]$. There is an unknown set $P\subset [r]$ such that $\forall i\in P$ there exists $\{S_{i_k}\}$ and $i = \cap_k S_{i_k}$. Our goal is to find the set $P$.

\begin{algorithm}
\begin{algorithmic}
\INPUT $k$ sets $S_1,...,S_h$.
\OUTPUT A set $P$ that has all the intersection vertices.
\STATE Initialize $P = \emptyset, R = \emptyset$.
\FOR{$i = 1$ to $r$}
\STATE Let $S = [r]$
\FOR{$j = 1$ to $h$}
\IF{$|S\cap S_j| < |S|$ and $S\cap S_j\not\subseteq R$}
\STATE $S = S\cap S_j$
\ENDIF
\ENDFOR
\STATE $R = R\cup S$
\STATE Add $S$ to $P$ if $|S| = 1$.
\ENDFOR
\end{algorithmic}
\caption{Finding Intersection}\label{alg:findintersectmeta}
\end{algorithm}

This problem is simple if we know the subsets of $W^j$ in each facet. However, since what we really have access to are subspaces, it is impossible to identify the vertices unless we have a subspace of dimension 1. On the other hand, we can perform intersection and linear-span for the subspaces, which correspond to intersection and union  for the sets. We also know the size of a set by looking at the dimension of the subspace. 
The main challenge here is that we cannot afford to enumerate all the possible combinations of the sets, and also there are vertices that are not intersection vertices and they may or may not appear in the sets we have.
The idea of the algorithm is to keep vertices that we have already found in $R$, and try to avoid finding the same vertices by making sure $S$ is never a subset of $R$. We show after every inner-loop one of the two cases can happen: in the first case we find an element in $P$; in the second case $S$ is a set that satisfies $(S\backslash R)\cap P = \emptyset$, so by adding $S$ to $R$ we remove some of the vertices that are not in $P$. Since the size of $R$ increases by at least $1$ in every iteration until $R = [r]$, the algorithm always ends in $r$ iterations and finds all the vertices in $P$.
In practice, we implement all the set operations in \ref{alg:findintersectmeta} using the analogous subspace operations (see Algorithm 6 in Appendix). We prove the following :

\begin{theorem}
\label{thm:findintersection}
When $W$ is $\alpha$-robust and $\epsilon_S < o(\alpha^3/r^{2.5})$, Algorithm 6 finds all the intersection vertices of $W$, with error at most $\epsilon_v = 4r^{1.5}\epsilon_S/\alpha$. 
\end{theorem}

\paragraph{Finding the remaining vertices}
\begin{algorithm}
\begin{algorithmic}
\INPUT matrix $\hat{M}$, intersection vertices $\hat{W}^1,...,\hat{W}^{|P|}$.
\OUTPUT remaining vertices $\hat{W}^{|P|+1},...,\hat{W}^r$.
\FOR{$i = |P|+1$ TO $r$}
\STATE Let $Q = \mbox{span}\{\hat{W}^1,...,\hat{W}^{i-1}\}$.
\STATE Pick the point $\hat{M}^j$ with largest $\|P_{Q^\perp} \hat{M}^j\|$, let $\hat{W}^i = \hat{M}^j$.
\ENDFOR
\end{algorithmic}
\caption{Finding remaining vertices} \label{alg:findremaining}
\end{algorithm}

The remaining vertices correspond to singleton sets in subset-separable assumption. They appear in rows of $M$. The situation is very similar to the separable NMF and we use an algorithm from \cite{AroraEtAl_icml13} to find the remaining vertices. For completeness we describe the algorithm here.
By Lemma 4.5 in \cite{AroraEtAl_icml13} we directly get the following theorem:

\begin{theorem}
\label{thm:findremaining}
If vertices already found have accuracy $\epsilon_v$ such that $\epsilon_v \le \alpha/20r$, Algorithm~\ref{alg:findremaining} outputs the remaining vertices with accuracy $O(\epsilon/\alpha^2) < \epsilon_v$.
\end{theorem}

\paragraph{Running time.} Face-Intersect (Algorithm 1) has 3 parts: find facets (Algorithm 3), find intersections (Algorithm 4) and find remaining anchors (Algorithm 5). We discuss the runtime of each part. We first do dimension reduction to map the $n$ points to an $r$-dimensional subspace to improve the running time of later steps. The dimension reduction takes $O(nmr)$ time, where $n,m$ are the number of rows and columns of $M$, respectively, and $r$ is the rank of the factorization.
Algorithm 3's runtime is $O(nd \cdot \mbox{OPT})$, where $d$ is the max dimension of properly filled facets (typically $d < r \ll m$). OPT is the time to solve the convex optimization problem in Algorithm 2. OPT is essentially equivalent to solving an LP with $n$ nonnegative variables and $r+d$ constraints. 
Algorithm 4's runtime is $O(k r^4)$ where $k$ is the number of properly filled facets; typically $k \ll n$.
Algorithm 5's runtime is $O(n r^3)$. The overall runtime of Face-Intersect is $O(mnr + nd\cdot  \mbox{OPT} + k r^4 + n r^3)$. Calling the OPT routine is the most expensive part of the algorithm. Empirically, we find that the algorithm converges after $\sim k \ll nd$ calls to OPT.

\section{Generative model of NMF naturally creates properly filled facets}
\label{sec:genmain}
To better understand the generality of our approach, we analyzed a simple generative model of subset-separable NMFs and showed that properly filled facets naturally arise with high probability. 

\paragraph{Generative Model}
Given a simplex $W$ that is $\alpha$-robust and a subset of facets $S_1,S_2,...,S_k$ that is subset separable. Let $p_i$ be the probability associated with facet $i$, and let $p_{min} = \min_{i\le k} p_i$ and $d = \max_{i\in [k]}|S_i|$. 
For convenience, denote $S_0 = [r]$ and $p_0 = 1 - \sum_{i=1}^k p_i$. To generate a sample, first sample facet $S_i$ with probability $p_i$, and then uniformly randomly sample a point within the convex hull of the points $\{W^j:j\in S_i\}$. 
Here we think of $d$ as a small constant or $O((\log n)/\log\log n)$ (in general $d$ can be much smaller than $r$). For example, separability assumption implies $d = 1$, and it is already nontrivial when $d = 2$.

\begin{theorem}
\label{thm:model}
Given $n = \Omega(\max\{(4d)^d\log (d/\eta), kr^2 \log (d/p_{min}\eta)\}/p_{min})$ samples from the model, with high probability the facets $S_1, ..., S_k$ are $(p_{min}n/2, 200r^{1.5}/p_{min}\alpha, \alpha^2/16d)$ properly filled.
\end{theorem}

The proof relies on the following two lemmas. The first lemma shows that once we have enough points in a simplex, then there is a center point with high probability.

\begin{lemma}
\label{lem:simplex}
Given $n = \Omega((4d)^d\log d/\eta)$ uniform points $v^1,v^2,...,v^n$ in a standard $d$-dimensional simplex (with vertices $e_1,e_2,...,e_d$), with probability $1-\eta$ there exists a point $v_i$ such that $v_i = \sum_{j\ne i} w_j v^j$ ($w_j\ge 0,\sum_{j\ne i}w_j = 1$), and $\sigma_{min}(\sum_{j\ne i} w_j (v^j)(v^j)^T) \ge 1/16d$.
\end{lemma}

The next lemma shows unless a subspace contains a properly filled facet, it cannot contain too many points in its neighborhood.

\begin{lemma}
\label{lem:subspace}
Given $n =  \Omega(d^2\log (d/p_{min}\eta)/p_{min})$ uniform points $v^1,v^2,...,v^n$ in a standard $d$-dimensional simplex (with vertices $e_1,e_2,...,e_d$), with probability $1-\eta$ for all matrices $A$ whose largest column norm is equal to $1$, there are at most $p_{min}n/4$ points with $\|Av^i\| \le p_{min}/200d$.
\end{lemma}

 \section{Experiments}
 
\begin{figure*}[t!]
\centering
\includegraphics[trim=0cm 0cm 1cm 0cm, width=0.8\textwidth]{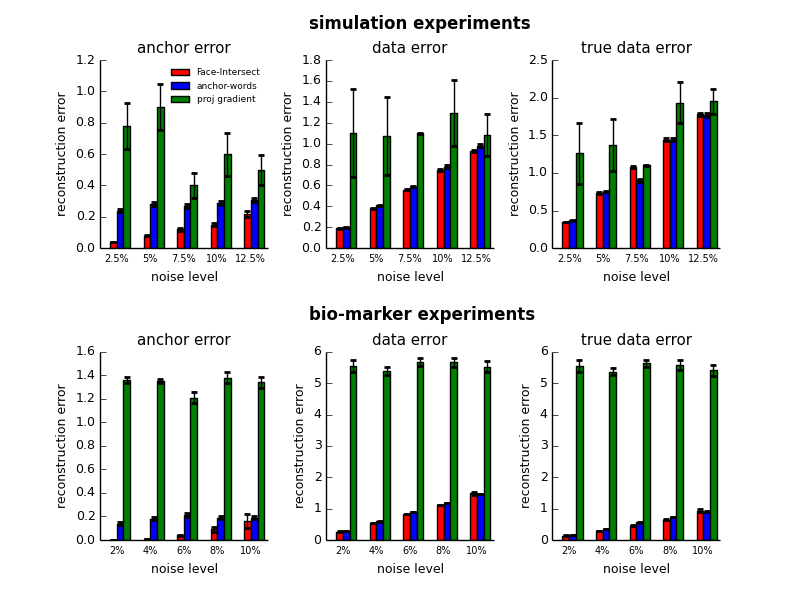}
\vspace{-0.7cm}
\caption{Reconstruction accuracy of the three NMF algorithm as a function of data noise. Standard error shown in the error bars. }\label{fig:plot}
\end{figure*} 
 
 While our algorithm has strong theoretical guarantees, we additionally performed proof-of-concept experiments to show that when the noise is relatively small, our algorithm can outperform the state-of-art NMF algorithms. 
 We simulated data according to the generative NMF model described in Section 7. We first randomly select $r$ non-negative vectors in $\mathbb{R}^m$ as rows of the $W$ matrix. We grouped the vertices $W^j$ into $r$ groups, $S_1, ..., S_r$ of three elements each, such that each vertex is the unique intersection of two groups. Each $S_i$ then corresponds to a 2-dim facet. To generate the $A$ matrix, for each $S_i$, we randomly sampled $n_1$ rows of $A$ with support $S_i$, where each entry is an i.i.d. from $\mbox{Unif}(0,1)$. An additional $n_2$ rows of $A$ were sampled with full support. These correspond to points in the interior of the simplex. We tested a range of settings with $m$ between 5 to 100, $r$ between 3 to 10, and $n_1$ and $n_2$ between 100 and 500. We generated the true data as $M = AW$ and added i.i.d. Gaussian noise to each entry of $M$ to generate the observed data $\tilde{M}$. 
 
 There are many algorithms for solving NMF, most of them are either iterative algorithms that have no guarantees, or algorithms that work only under separability condition. We choose two typical algorithms: the Anchor-Words algorithm\citep{AroraEtAl_icml13} for separable NMF, and Projected Gradient \citep{lin2007projected} for iterative algorithms. 
 For each simulated NMF, we evaluated the output factors $\hat{A}, \hat{W}$ of these algorithms on three criteria: accuracy of the reconstructed anchors to the true anchors, $||W - \hat{W}||_2$; accuracy of the reconstructed data matrix to the \textit{observed} data, $||\tilde{M} - \hat{A}\hat{W}||_2$; accuracy of the reconstructed data to the \textit{true} data, $||M - \hat{A}\hat{W}||_2$. 
 In Figure~\ref{fig:plot}, we show the results for the three methods under the setting $n_1 = 100, n_2 = 100, m = 10, r = 5$. We grouped the results by the noise level of the experiment, which is defined to be the ratio of the average magnitude of the noise vectors to the average magnitude of the data points in $\mathbb{R}^m$. Face-Intersect is substantially more accurate in reconstructing the $W$ matrix compared to Anchor-Words and Projected Gradient. In terms of reconstructing the $\tilde{M}$ and $M$ matrices, Face-Intersect slightly outperforms Anchor-Words ($p < 0.05$ t-test), and they both were substantially more accurate than Projected Gradient. As noise level increased, the accuracy of Face-Intersect and Anchor-Words degrades and at noise around $12.5\%$, the accuracy of the three methods converged. 
 In many applications, we are more interested in accurate reconstruction of the latent $W$ than of $M$. For example, in bio-medical applications, each row of $M$ is a sample and each column is the measurement of that sample at a particular bio-marker. Each sample is typically a mixture of $r$ cell-types, and each cell-type corresponds to a row of $W$. The $A$ matrix gives the mixture weights of the cell-types into the samples. Given measurement on a set of samples, $M$, an important problem is to infer the values of the latent cell-types at each bio-marker, $W$ \citep{ZOU}. To create a more realistic simulation of this setting, we used DNA methylation values measures at 100 markers in 5 cell-types (Monocytes, B-cells, T-cells, NK-cells and Granulocytes) as the true W matrix \citep{ZOU}. From these 5 anchors we generated 600 samples--which is a typical size of such datasets--using the same procedure as above. Both Face-Intersect and Anchor-Words substantially outperformed Projected Gradient across all three reconstruction criteria. In terms of reconstructing the biomarker matrix $W$, Face-Intersect was significantly more accurate than Anchor-Words. For reconstructing the data matrices $M$ and $\tilde{M}$, Face-Intersect was statistically more accurate than Anchor-Words when the noise is less than $8\%$ ($p < 0.05$), though the magnitude of the difference is small. 
 
\paragraph{Discussion}
We have presented the notion of subset separability, which substantially generalizes separable NMFs and is a necessary condition for the factorization to be unique or to have minimal volume. This naturally led us to develop the Face-Intersect algorithm, and we showed that when the NMF is subset separable and have properly filled facets, this algorithm provably recovers the true factorization. Moreover, it is robust to small adversarial noise. We show that the requirements for Face-Intersect to work are satisfied by simple generative models of NMFs. The original theoretical analysis of separable NMF led to a burst of research activity. Several highly efficient NMF algorithms were inspired by the theoretical ideas. We are hopeful that the idea of subset-separability will similarly lead to practical and theoretically sound algorithms for a much larger class of NMFs. Our Face-Intersect algorithm and its analysis is a first proof-of-concept that this is a promising direction. In exploratory experiments, we showed that under some settings where the relative noise is low, the Face-Intersect algorithm can outperform state-of-art NMF solvers. An important agenda of research will be to develop more robust and scalable algorithms motivated by our subset-separability analysis. 
        
\clearpage
\bibliographystyle{icml2015}
\bibliography{nmf}

\begin{thebibliography}{22}
\providecommand{\natexlab}[1]{#1}
\providecommand{\url}[1]{\texttt{#1}}
\expandafter\ifx\csname urlstyle\endcsname\relax
  \providecommand{\doi}[1]{doi: #1}\else
  \providecommand{\doi}{doi: \begingroup \urlstyle{rm}\Url}\fi

\bibitem[Arora et~al.(2012)Arora, Ge, Kannan, and Moitra]{AGKM}
Arora, S., Ge, R., Kannan, R., and Moitra, A.
\newblock Computing a nonnegative matrix factorization -- provably.
\newblock In \emph{STOC}, pp.\  145--162, 2012.

\bibitem[Arora et~al.(2013)Arora, Ge, Halpern, Mimno, Moitra, Sontag, Wu, and
  Zhu]{AroraEtAl_icml13}
Arora, Sanjeev, Ge, Rong, Halpern, Yoni, Mimno, David~M., Moitra, Ankur,
  Sontag, David, Wu, Yichen, and Zhu, Michael.
\newblock A practical algorithm for topic modeling with provable guarantees.
\newblock In \emph{Proceedings of the International Conference on Machine
  Learning (ICML)}, volume 28 (2), pp.\  280--288. JMLR: W\&CP, 2013.

\bibitem[Bittorf et~al.(2012)Bittorf, Recht, Re, and Tropp]{BRRT}
Bittorf, V., Recht, B., Re, C., and Tropp, J.
\newblock Factoring nonnegative matrices with linear programs.
\newblock In \emph{NIPS}, 2012.

\bibitem[Devarajan(2009)]{NMF_bio}
Devarajan, K.
\newblock Nonnegative matrix factorization: an analytical and interpretive tool
  in computational biology.
\newblock \emph{PLoS Comput Biol}, 2009.

\bibitem[Donoho \& Stodden(2003)Donoho and Stodden]{DS}
Donoho, D. and Stodden, V.
\newblock When does non-negative matrix factorization give the correct
  decomposition into parts?
\newblock In \emph{NIPS}, 2003.

\bibitem[Elhamifar \& Vidal(2009)Elhamifar and Vidal]{elhamifar2009sparse}
Elhamifar, Ehsan and Vidal, Ren{\'e}.
\newblock Sparse subspace clustering.
\newblock In \emph{Computer Vision and Pattern Recognition, 2009. CVPR 2009.
  IEEE Conference on}, pp.\  2790--2797. IEEE, 2009.

\bibitem[Gillis(2012)]{G}
Gillis, N.
\newblock Robustness analysis of hotttopixx, a linear programming model for
  factoring nonnegative matrices.
\newblock 2012.
\newblock http://arxiv.org/abs/1211.6687.

\bibitem[Gillis(2014)]{NMFsurvey}
Gillis, N.
\newblock The why and how of nonnegative matrix factorization, 2014.
\newblock http://arxiv.org/abs/1401.5226.

\bibitem[Gillis \& Vavasis(2014)Gillis and Vavasis]{GillisVavasis2}
Gillis, N. and Vavasis, S.A.
\newblock Fast and robust recursive algorithmsfor separable nonnegative matrix
  factorization.
\newblock \emph{Pattern Analysis and Machine Intelligence, IEEE Transactions
  on}, 36\penalty0 (4):\penalty0 698--714, April 2014.
\newblock ISSN 0162-8828.
\newblock \doi{10.1109/TPAMI.2013.226}.

\bibitem[Gomez et~al.(2007)Gomez, Borgne, Allemand, Delacourt, and
  Ledru]{NFINDR}
Gomez, C., Borgne, H.~Le, Allemand, P., Delacourt, C., and Ledru, P.
\newblock N-findr method versus independent component analysis for lithological
  identification in hyperspectral imagery.
\newblock \emph{Int. J. Remote Sens.}, 28\penalty0 (23), January 2007.

\bibitem[Hardt \& Moitra(2013)Hardt and Moitra]{HardtM13}
Hardt, Moritz and Moitra, Ankur.
\newblock Algorithms and hardness for robust subspace recovery.
\newblock In \emph{COLT}, pp.\  354--375, 2013.

\bibitem[Kumar et~al.(2012)Kumar, Sindhwani, and Kambadur]{Kumar12}
Kumar, A., Sindhwani, V., and Kambadur, P.
\newblock Fast conical hull algorithms for near-separable non-negative matrix
  factorization.
\newblock 2012.
\newblock http://arxiv.org/abs/1210.1190v1.

\bibitem[Lee \& Seung(1999)Lee and Seung]{LS99}
Lee, D. and Seung, H.
\newblock Learning the parts of objects by non-negative matrix factorization.
\newblock \emph{Nature}, pp.\  788--791, 1999.

\bibitem[Lee \& Seung(2000)Lee and Seung]{LS00}
Lee, D. and Seung, H.
\newblock Algorithms for non-negative matrix factorization.
\newblock In \emph{NIPS}, pp.\  556--562, 2000.

\bibitem[Lin(2007)]{lin2007projected}
Lin, Chih-Jen.
\newblock Projected gradient methods for nonnegative matrix factorization.
\newblock \emph{Neural computation}, 19\penalty0 (10):\penalty0 2756--2779,
  2007.

\bibitem[Moitra(2013)]{DBLP:conf/soda/Moitra13}
Moitra, Ankur.
\newblock An almost optimal algorithm for computing nonnegative rank.
\newblock In \emph{Proceedings of the Twenty-Fourth Annual {ACM-SIAM} Symposium
  on Discrete Algorithms, {SODA} 2013, New Orleans, Louisiana, USA, January
  6-8, 2013}, pp.\  1454--1464, 2013.

\bibitem[Nascimento \& Dias(2004)Nascimento and Dias]{VCA}
Nascimento, J.M.~P. and Dias, J. M.~B.
\newblock Vertex component analysis: A fast algorithm to unmix hyperspectral
  data.
\newblock \emph{IEEE TRANS. GEOSCI. REM. SENS}, 43:\penalty0 898--910, 2004.

\bibitem[Stewart \& Sun(1990)Stewart and Sun]{SS90}
Stewart, G.W. and Sun, J.
\newblock \emph{Matrix perturbation theory}, volume 175.
\newblock Academic press New York, 1990.

\bibitem[Vavasis(2009)]{Vav}
Vavasis, S.
\newblock On the complexity of nonnegative matrix factorization.
\newblock \emph{SIAM Journal on Optimization}, pp.\  1364--1377, 2009.

\bibitem[Vidal(2010)]{vidal2010tutorial}
Vidal, Ren{\'e}.
\newblock A tutorial on subspace clustering.
\newblock \emph{IEEE Signal Processing Magazine}, 28\penalty0 (2):\penalty0
  52--68, 2010.

\bibitem[Xu et~al.(2003)Xu, Liu, and Gong]{XLG}
Xu, W., Liu, X., and Gong, Y.
\newblock Document clustering based on non-negative matrix factorization.
\newblock In \emph{SIGIR}, pp.\  267--273, 2003.

\bibitem[Zou et~al.(2014)Zou, Lippert, Heckerman, Aryee, and Listgarten]{ZOU}
Zou, J., Lippert, C., Heckerman, D., Aryee, M., and Listgarten, J.
\newblock Genome-wide association studies without the need for cell-type
  composition.
\newblock \emph{Nature Methods}, pp.\  309--311, 2014.

\end{thebibliography}
\clearpage
\appendix
\section{Subset Separability and minimal volume}

In this section we prove Proposition~\ref{prop:volume} subset separability condition is necessary for a minimal volume solution.

\begin{proof}
Suppose $M = AW$ is a rank-$r$ nonnegative matrix factorization with minimal volume. If this decomposition does not satisfy the subset-separable condition, then there exists $i\ne j\in R$ such that for every row $A^t$, the two entries $A_{t,i}, A_{t,j}$ are either all zero or all nonzero. That is, the columns $A_i$ and $A_j$ have the same support. Consider a new factorization $A'W'$, where the columns of $A'$ are the same as columns of $A$ except for columns $i,j$, and rows of $W'$ are the same as rows of $W'$ except for row $i$.

Let $A'_i = \frac{1}{1-\epsilon} A_i$, $A'_j = A_j - \frac{\epsilon}{1-\epsilon} A_i$, and $(W')^i = (1-\epsilon) W^i + \epsilon W^j$, it is easy to verify that $A'W' = AW = M$, and $W'$ is still nonnegative for $\epsilon \in [0,1]$.

Since the support of $A_i$ and $A_j$ are the same, there exists a positive $\epsilon$ such that $A'_j$ is still a nonnegative vector. In that case $A'W'$ is a valid nonnegative matrix factorization where only one row of $W'$ is different from $W$. By construction it is clear that the volume of $W'$ is equal to $(1-\epsilon)$ times the volume of $W$, so this contradicts with the assumption that $M = AW$ is a factorization with minimal volume.
 \end{proof}

\section{Detailed analysis for finding properly filled facets}
\label{sec:appendix:facet}
In this section we analyze Algorithms~\ref{alg:findone} and \ref{alg:findall}.

\subsection{Finding one properly filled facet}

We first prove Theorem~\ref{thm:findonefacet}. For this algorithm, it is more natural to use the following robustness condition, which is a corollary of $\alpha$-robustness.

\begin{lemma}
Suppose the vertices of the unknown simplex are rows of $W\in \mathbb{R}^{r\times n}$, and $W$ is $\alpha$-robust. For any face $S$ of $W$ with corresponding subspace $Q$, there exists a unit vector $h\perp Q$, let $v$ be any vector in the simplex and $v^\perp$ be its component that is orthogonal to $Q$, then $\frac{|h\cdot v^\perp|}{\|v^\perp\|} \ge \frac{\alpha}{\sqrt{r}}$. 
\end{lemma}

\begin{proof}
Suppose $Q$ has dimension $d$ (we know $d < r$). Let $B$ be the projection of $W$ to the orthogonal subspace of $Q$, and remove the 0-columns in $B$. The matrix $B$ is a $n\times (r-d)$ matrix whose smallest singular value is at least $\alpha$ (the smallest singular value in a projection is at least the smallest singular value of the matrix). We construct $h$ as $h = \frac{B^\dag\vec{1}}{\|B^\dag\vec{1}\|}$. By the property of $B$ we know $h\cdot B_i = \frac{1}{\|B^\dag\vec{1}\|} = \alpha/\sqrt{r}$.

For any vector $v$ in the simplex, its orthogonal component $v^\perp$ is equal to $P_{Q^\perp} (\sum_{i=1}^r w_i W^i)$, which is a nonnegative combination of columns in $B$. Therefore $\frac{|h\cdot v^\perp|}{\|v^\perp\|} = \frac{\sum_i w_i h_i\cdot B_i}{\|\sum_i w_i B_i\|} \ge \max_i \frac{h_i\cdot B_i}{\|B_i\|} = \alpha/\sqrt{r}$ (here we used the fact that $h\cdot B_i$ are all positive).
\end{proof}

As we explained, there are two challenges in proving Theorem~\ref{thm:findonefacet}: 1). the observations are noisy. We would like to show even with the noisy $\hat{v}$'s, the subspace $\hat{Q}$ is always close to a subspace of the true space $Q$; 2). the convex combination may not find the entire space $Q$, for which we show the dimension of $\hat{Q}$ will increase until it is equal to the dimension of $Q$. Throughout this section we will use $d$ to denote the dimension of true space $Q$.

We first show that in every step of the algorithm all the vectors in $\hat{Q}$ are close to the subspace $Q$. We start by proving a general perturbation lemma for singular subspaces:

\begin{lemma}
\label{lem:spaceperturbation}
Let $\hat{F} = F + E$ where both $\hat{F}$ and $F$ are positive semidefinite, $F$ is a rank $d$ matrix with column span $Q$, and $U$ is the top $t$ ($t \le d$) singular space of $\hat{F}$ with the $t$-th singular value $\sigma_t(\hat{F}) > \|E\|$, then $\| P_{Q^\perp} U\| \le \|E\|/\sigma_t(\hat{F})$.
\end{lemma}
\begin{proof}
Let $UDU^T$ be the truncated top $t$ SVD of $\hat{F}$, and $\hat{U}\hat{D}\hat{U}^T$ be the full SVD. We know $$\|P_{Q^\perp} \hat{U}\hat{D}\| = \|P_{Q^\perp} \hat{F}\| = \|P_{Q^\perp} E\| \le \|E\|,$$ 
where the first equality is because $\hat{U}$ is an orthonormal matrix, and the second equality is because the column span of $F$ is inside $Q$. 

On the other hand, $P_{Q^\perp} UD$ is a submatrix of $P_{Q^\perp} \hat{U}\hat{D}$, so we know $\|P_{Q^\perp} UD\| \le \|P_{Q^\perp} \hat{U}\hat{D}\| \le \|E\|$. Since all the entries in $D$ are at least $\sigma_t(\hat{F})$, this implies $\| P_{Q^\perp} U\| \le \|E\|/\sigma_t(\hat{F})$.
\end{proof}

In the later proofs we usually think of $\hat{F}$ as $\left(\sum_{i=1}^n w_i \hat{v}^i(\hat{v}^i)^T\right)$, and $F$ as $\left(\sum_{i=1}^n w_i  P_Q v^i (v^i)^TP_Q\right)$. The next lemma shows that for any feasible solution of the optimization program (even just considering the first three constraints), the matrix $\hat{F}$ is close to $F$:
\begin{lemma}
\label{lem:feasible}
For any feasible solution that satisfies the first three constraint, let $\hat{F} = \left(\sum_{i=1}^n w_i \hat{v}^i(\hat{v}^i)^T\right)$ and $F = \left(\sum_{i=1}^n w_i P_Qv_iv_i^TP_Q\right)$, we have $\|E\| = \|\hat{F}-F\| \le O(\sqrt{r}\epsilon/\alpha)$. In fact, even the nuclear norm\footnote{Nuclear norm $\|M\|_*$ is equal to the sum of singular values of $M$, it is also the dual norm of spectral norm in the sense that $\|M\|_* = \max_{\|A\|\le 1} \langle A,M\rangle$.}  $\|\hat{F} - F\|_* \le O(\sqrt{r}\epsilon/\alpha)$.
\end{lemma}
\begin{proof}
Let $\tilde{F} = \left(\sum_{i=1}^n w_i v^i (v^i)^T\right)$, we show $\tilde{F}$ is close to both $F$ and $\hat{F}$.
Let $\delta_i = \hat{v}^i - v^i$. By assumption we know $\|\delta_i\| \le \epsilon$. Also, by assumption $\|v^i\| \le 1$ (normalization) so $\|\hat{F} - \tilde{F}\| \le \sum_{i=1}^n w_i \|\delta_i (v^i)^T+v^i\delta_i^T + \delta_i\delta_i^T\| \le (2\epsilon+\epsilon^2) \sum_{i=1}^n w_i =  O(\epsilon)$. 

On the other hand, by the third constraint we know $\|\hat{v}^0 - \sum_{i=1}^n w_i \hat{v}^i\| \le 2\epsilon$, which implies $\|v^0 - \sum_{i=1}^n w_i v^i\| \le 4\epsilon$ (because $\|v^i - \hat{v}^i\| \le \epsilon$ and $w_i$'s form a probability distribution). Using the robustness condition, let $v^{i\perp} = v^i - P_Qv_i$, then
\[
4\epsilon > \| \sum_i w_i v^i - v^0\| \geq \sum_i w_i h_I \cdot v^{i\perp} \geq \bigtriangleup \sum_{i=1}^n w_i\|v^{i\perp}\|
\]
Therefore we know $\|\tilde{F} - F\| \le \sum_{i=1}^n w_i \|v^{i\perp} (v^i)^TP_Q +  P_Qv_i(v^{i\perp})^T + (v^{i\perp})(v^{i\perp})^T\| \le O(\sqrt{r}\epsilon/\alpha)$ (note that $\|v^{i\perp}\| \le 1$ by normalization).  

The nuclear norm bound follows from exactly the same proof.
\end{proof}

The previous two lemmas guarantee that at any time of the algorithm, the subspace $\hat{Q}$ is always close to a subspace of $Q$. In the next lemma we show that the algorithm makes progress

\begin{lemma}
\label{lem:spaceprogress}
If $dim(\hat{Q}) = t < d$, then in the next iteration the dimension of $\hat{Q}$ increases by at least $1$.
\end{lemma}

\begin{proof}
Since $v^0$ is a center of the facet, we know there exists a ``nice'' solution $w^*$ such that $v^0 = \sum_{i=1}^n w^*_i v^i$ and $\sum_{i=1}^n w^*_i (v^i)(v^i)^T$ has $d$-th singular value $\gamma$.

We first show that this guaranteed good solution $w^*$ is always a feasible solution. Clearly it satisfies the first three constraints (by triangle inequality). For the last constraint, let $F^*$ be the $F$-matrix constructed be $w^*$ and $\hat{F}^*$ be the corresponding $\hat{F}$ matrix. By assumption we know $\hat{F}^*$ has $\sigma_d(\hat{F}^*) \ge \gamma$, so in particular for any direction $u$ in subspace $Q$, $u^T F^*u \ge \gamma$. Since by previous two lemmas we have $\|P_{Q^\perp} \hat{Q}\| \le O(d\sqrt{r}\epsilon/\alpha\gamma)$, in particular every column of $\hat{Q}$ is within $O(d\sqrt{r}\epsilon/\alpha\gamma)$ with its projection in $Q$, we know $\mbox{diag}(\hat{Q}^T F^* \hat{Q}) \ge \frac{3}{4}\gamma$ (when $\sqrt{r}\epsilon/\alpha\gamma$ is smaller than some universal constant). Now by Lemma~\ref{lem:feasible} $F^*$ and $\hat{F}^*$ are close (in spectral norm) we have $\mbox{diag}(\hat{Q}^T \hat{F}^* \hat{Q}) \ge \gamma/2$.

Since the solution $w^*$ is feasible, the optimal solution must have objective value no less than the objective value of $w^*$. By the nuclear norm bound, for any subspace $\hat{Q}$ we know $$\mbox{tr}(P_{\hat{Q}^\perp}FP_{\hat{Q}^\perp}) - \mbox{tr}(P_{\hat{Q}^\perp}FP_{\hat{Q}^\perp}) = \mbox{tr}(P_{\hat{Q}^\perp}(F-\hat{F})P_{\hat{Q}^\perp}) \le \|P_{\hat{Q}^\perp}(F-\hat{F})P_{\hat{Q}^\perp}\|_* \le \|F-\hat{F}\|_* \le O(\sqrt{r}\epsilon/\alpha),$$
where we used the fact that the trace of a matrix is always bounded by its nuclear norm, and nuclear norm of a projection is always smaller than nuclear norm of the original matrix  \footnote{This follows from the fact that $\|A\|_* = \max_{\|B\| \le 1} \langle A,B\rangle$ and spectral norms do not increase after projection.}.

On the other hand $\mbox{tr}(P_{\hat{Q}^\perp}F^*P_{\hat{Q}^\perp}) \ge \mbox{tr}(F^*) - \sum_{i=1}^t \sigma_i(F^*) \ge \sum_{i=t+1}^d \sigma_i(F^*) \ge \gamma (d-t)$. So the optimal objective value must be at least $\gamma(d-t) - O(\sqrt{r}\epsilon/\alpha)$.

Let $w$ be the optimal solution and $F$, $\hat{F}$ be the corresponding matrices, by the same argument we know 
 $$\mbox{tr}(P_{\hat{Q}}^\perp FP_{\hat{Q}}^\perp) \ge \mbox{tr}(P_{\hat{Q}}^\perp \hat{F}P_{\hat{Q}}^\perp) - O(\sqrt{r}\epsilon/\alpha) \ge \gamma(d-t) - O(\sqrt{r}\epsilon/\alpha).$$  
 However, $F$ is a matrix of rank at most $d$, therefore $\|P_{\hat{Q}}^\perp FP_{\hat{Q}}^\perp\| \ge \gamma/d - O(\sqrt{r}\epsilon/\alpha d)$. For the $\hat{F}$ matrix we also have $\|P_{\hat{Q}}^\perp \hat{F}P_{\hat{Q}}^\perp\| \ge \gamma/2d$ because $\|\hat{F}-F\|$ is small.

Now for the matrix $\hat{F}$, there are $t+1$ orthogonal directions ($t$ from $\hat{Q}$, and at least one orthogonal to $\hat{Q}$) with singular value at least $\gamma/2d$, hence
$\sigma_{t+1}(\hat{F}) \ge \gamma/2d$. As a result in the next step the dimension of $\hat{Q}$ increases by at least 1.
\end{proof}

Now we are ready to prove Theorem~\ref{thm:findonefacet}.

\begin{proof} (of Theorem~\ref{thm:findonefacet}).
By Lemma~\ref{lem:spaceprogress} and Lemma~\ref{lem:feasible}, we know when the algorithm ends we must have $\mbox{dim}(\hat{Q}) \ge \mbox{dim}(Q)$.

Now by the last constraint, we know $\sigma_r(\sum_{i=1}^n w_i \hat{v}^i(\hat{v}^i)^T) \ge \gamma/2$. Combined with Lemma~\ref{lem:feasible} and Lemma~\ref{lem:spaceperturbation} this implies the final subspace is within distance $O(\sqrt{r}\epsilon/\alpha \gamma)$.
\end{proof}

\subsection{Finding all properly filled facets}

Theorem~\ref{thm:findall} follows immediately from Theorem~\ref{thm:findonefacet}, for completeness we provide the proof here.

\begin{proof} (of Theorem~\ref{thm:findall})
By Theorem~\ref{thm:findonefacet}, and by Condition 2 in Definition~\ref{def:robust}, when we run Algorithm~\ref{alg:findone} on a correct center point, the resulting subspace will always be added to the collection. Therefore at the end of the loop for each facet $S_i$ with at least two vertices, and its corresponding subspace $Q_i$, there must be a $\hat{Q}_i$ in the collection that is $O(\sqrt{r}\epsilon/\alpha\gamma)$-close.

On the other hand, by Condition 3 in Definition~\ref{def:robust} we know every subspace $\hat{Q}$ that is in the collection must satisfy $\|P_{\hat{Q}^\perp} Q_i\| \le O(H\sqrt{r}\epsilon/\alpha\gamma)$ for some true subspace $Q_i$. If $\hat{Q}$ has dimension larger than $Q_i$, then $\|P_{\hat{Q}^\perp} \hat{Q}_i\| \le \|P_{\hat{Q}^\perp} Q_i\| + \|P_{Q^\perp} \hat{Q}_i\| \le O(H\sqrt{r}\epsilon/\alpha\gamma)$\footnote{This uses the variational characterization of $P_{U^\perp} V = \max_{u\in U, v\in V} \sin \theta(u,v)$ where $\theta(u,v)$ is the angle between $u,v$.}. Therefore all the false positives with higher dimension are removed. The remaining subspaces must be $O(H\sqrt{r}\epsilon/\alpha\gamma)$-close to one of the true subspaces.

By the $\alpha$-robustness condition, two subspaces corresponding to different facets must have distance at least $\alpha$, so when $H\sqrt{r}\epsilon/\alpha\gamma \le o(\alpha)$ the subspaces $\hat{Q}$ close to a true subset cannot be removed. Also, in the last step it is easy to identify the subspaces $\hat{Q}$ that are close to one true space $Q_i$, any one of those will be $\epsilon_S$-close to the true subsets.
\end{proof}

\section{Detailed Analysis for finding intersections}
\label{sec:appendix:intersect}
In this section we first prove Theorem~\ref{thm:findintersection}, then we discuss how to apply Algorithm~\ref{alg:findremaining} from \cite{AroraEtAl_icml13} to find the remaining vertices.

\begin{algorithm}
\begin{algorithmic}
\INPUT $k$ subspaces $\hat{Q}_1,...,\hat{Q}_h$.
\OUTPUT intersection vertices $\hat{W}^i$ that corresponds to $\{W^i:i\in P\}$
\STATE Let $\epsilon_v = 4r^{1.5}\epsilon_S/\alpha$, $\epsilon_Y = 2r\epsilon_v/\alpha$.
\STATE Maintain list of vertices $\{\hat{W}^i\}$, matrix $Y$, and subspace $\Gamma$ that correspond to the left singular space of $Y$ with singular values larger than $\alpha/2$.
\FOR {$i = 1$ TO $r$}
\STATE Maintain set $U \subset [k]$, $\Sigma = [\hat{Q}_i^\perp:i\in U]$ and $Z$ be the space of left singular vectors of $\Sigma$ with singular values at most $r\epsilon_S$. 
\STATE Initialize $S = \emptyset$
\FOR {$j  = 1$ TO $h$}
\STATE Let $\Sigma' = \Sigma+P_{\hat{Q}_j^\perp}$, Let $Z'$ be the space of eigenvectors of $\Sigma$ with eigenvalues at most $r\epsilon_S$.
\IF {$\mbox{dim}(Z') < \mbox{dim}(Z)$ and  $\|P_{\Gamma^\perp} Z'\| > \alpha/2$}
\STATE let $U = U\cup \{j\}$, replace $\Sigma$ and $Z$ by $\Sigma',Z'$.
\ENDIF
\ENDFOR
\STATE Append $Z$ to $Y$ ($Y = [Y, Z]$), update $\Gamma$.
\STATE If $\mbox{dim}(Z) = 1$ then add the direction to list of $\hat{W}$
\ENDFOR
\end{algorithmic}
\caption{Finding Intersection}\label{alg:intersect}
\end{algorithm}

The main idea of the implementation is that the subspace $Z$ will always be close to the span of $\{W^i:i\in S\}$ where $S = \cap_{j\in U} S_j$. The subspace $\Gamma$ will correspond to span of $\{W^i:i\in R\}$ where $R$ is the set of points that we have already found. If $\|P_{\Gamma^\perp}Z\|$ is large then it means $S$ is not a subset of $R$.

For this step we also need a particular corollary of the $\alpha$-robustness condition. 
\begin{lemma}
\label{lem:sumofprojection}
Suppose the vertices of the unknown simplex are rows of $W\in \mathbb{R}^{r\times m}$ and $W$ is $\alpha$-robust. Let $Q_1,Q_2,..., Q_t$ be a set of faces that has intersection $S\subset [r]$, and $Q_i^\perp$ be an arbitrary basis for the orthogonal subspace of $Q_i$. The matrix $\Sigma = [Q_1^\perp, Q_2^\perp, ..., Q_t^\perp]^T$ has a null-space equal to $\mbox{span}\{W^i:i\in S\}$ and $\sigma_{n-|S|}(\Sigma) \ge \alpha/\sqrt{r}$.
\end{lemma}

\begin{proof}
Clearly all the vectors $\{W^i:i\in S\}$ are in the null-space of $\Sigma$ as $W^i\in Q_j$ for all $j\in[t]$. For vectors that are orthogonal to the span of columns of $W$, they have projection $1$ in all of $P_{Q_j^\perp}$'s, and they do not influence the projections within the row span of $W$. We only need to prove that within the row span of $W$, for all the directions orthogonal to $\{W^i:i\in S\}$ the matrix still has large singular values.

Let $S_j$ be the set of vertices that $Q_j$ contains, we define $S_j'$ as follows: $S_1' = [r]\backslash S_1$, for all $j >1$ $S'_j = [r]\backslash \left((\cup_{j' < j} S'_{j'})\cup S_j\right)$. Since $S$ is the intersection of the verticies, we know $\cup S'_j = [r]\backslash S$. Also by construction we know the $S'_j$'s are disjoint. For each $S'_j$, let $Q'_j$ be the span of rows of $W$ with indices in $[r]\backslash S'_j$. Since $[r]\backslash S'_j$ is a superset of $S_j$, we know $Q_j$ is a subspace of $Q'_j$ and hence $P_{(Q'_j)^\perp} \preceq P_{Q_j^\perp}$. For each $j$ construct $B^j$ to be the matrix that is an (arbitrary) orthogonal basis of the orthogonal subspace of $Q'_j$ in span of $W$. Let $B = [B^1, B^2,..., B^t] \in \mathbb{R}^{n\times (r-1)}$. We know $BB^T \preceq \sum_j P_{(Q'_j)^\perp} \preceq \Sigma \Sigma^T$. Therefore we only need to show the matrix $B$ has large smallest singular value.

Now consider the product $WB$. By construction of $B$, this is a block diagonal matrix (with blocks correspond to $S'_j$'s). Since $W$ is $\alpha$-robust we know each block has smallest singular value $\alpha$. Therefore $\sigma_{min}(WB) \ge \alpha$, and $\sigma_{min}(B) \ge \alpha/\|W\| \ge \alpha/\sqrt{r}$. By the relationship between $\Sigma$ and $B$ we know $\sigma_{n-1}(\Sigma)\ge \sigma_{min}(B) \ge \alpha/\sqrt{r}$.
\end{proof}

This lemma allows us to take the intersections of subspaces robustly.

We prove the theorem by induction. The induction hypothesis is

\begin{claim}
At the end of every outer-loop, $\hat{W}^i$'s are $\epsilon_v = 4r^{1.5}\epsilon_S/\alpha$ close to some vertices in $W$, $\Gamma$ is $\epsilon_Y = 2r\epsilon_v/\alpha$-close to a subspace spanned by $W^R$ where $R\subset [r]$ is a subset of vertices. The set $R$ never contains any vertex in $P$ that is not already close to one of the elements in the list $\hat{W}^i$.
\end{claim}

Clearly this hypothesis is true before the first iteration (everything was empty). Next we analyze the inner-loop of the algorithm. During the inner-loop the algorithm maintains the following properties:

\begin{lemma}
\label{lem:spaceclose}
The set $U$ always has size at most $r-1$, the subspace $Z$ is always $\epsilon_v = 4r^{1.5}\epsilon_S/\alpha$-close to the subspace spanned by $\{A_i:i\in \cap_{j\in U} S_j\}$.
\end{lemma}

\begin{proof}
After the first element is added to $U$, the dimension of $Z$ is equal to the dimension of some $\hat{Q}_j$, which is at most $r-1$. Every time we add an element to $U$ the dimension of $Z$ decreases by $1$, and when $\mbox{dim}(Z)$ becomes 1 the algorithm stops. So there must be at most $r-1$ elements in $U$. By Lemma~\ref{lem:sumofprojection} we know if the matrix consist of the true $Q_j^\perp$, then it has nullspace equal to the span of $\{W^i:i\in \cap_{j\in U} S_j\}$, and all the other directions have eigenvalue at least $\alpha/\sqrt{r}$. The difference between $\Sigma$ and the true matrix is at most $2r\epsilon_S$, so when $2r\epsilon_S < \alpha/4\sqrt{r}$ by matrix perturbation bounds (Wedin's Theorem\cite{SS90}) we know $Z$ is always $4r^{1.5}\epsilon_S/\alpha$-close.
\end{proof}

Another property is that in the intersection $\cap_{j\in U} S_j$ there is always an element that is not already found.

\begin{lemma}
\label{lem:outside}
$\|P_{\Gamma^\perp}Z\| > \alpha/2$ if and only if 
$\cap_{j\in U} S_j$ contains at least one element outside of $R$. Further, this ensures $S = \cap_{j\in U} S_j$ always contains at least one element outside of $R$ during the inner-loop.
\end{lemma}

\begin{proof}
This is because by induction hypothesis $\Gamma$ is $\epsilon_Y$-close to the row span of $W^R$. On the other hand by Lemma~\ref{lem:spaceclose} we know $Z$ is $\epsilon_v$ close to the row span of $W^{\cap_{j\in U} S_j}$. If $\cap_{j\in U} S_j \subset R$ then the row span of $W^{\cap_{j\in U} S_j}$ is a subspace of the row span of $W^R$, and $\|P_{\Gamma^\perp}Z\| \le \epsilon_Y+\epsilon_v \ll \alpha/2$.

On the other hand, if $\cap_{j\in U} S_j$ has an element that is outside $R$, then since $W$ is $\alpha$-robust, there is a direction in $W^{\cap_{j\in U} S_j}$ that has distance at least $\alpha$ to the row span of $W^R$. By triangle inequality the distance between $P_{\Gamma^\perp}Z \ge \alpha - \epsilon_Y-\epsilon_v > \alpha/2$.

The last statement of the lemma then follows directly because this is true initially ($S = [r]$ initially) and the conditions in the if-statement ensures this property is preserved.
\end{proof}

Using these two properties we know whenever the inner-loop adds a point to the list $\hat{W}^i$ then it must be $\epsilon_v$ close to one of the unfound $W^i$'s (which is the first part of the induction hypothesis). Next we prove if at the end of the inner-loop $\mbox{dim}(Z)$ is more than 1, then $\cap_{j\in U} S_j$ does not contain any vertices in $P\backslash R$.

\begin{lemma}
If $\mbox{dim}(Z)$ is more than $1$ after the inner-loop, then $\cap_{j\in U} S_j$ does not contain any vertices in $P\backslash R$. 
\end{lemma}

\begin{proof}
Assume towards contradiction that $\mbox{dim}(Z) > 1$ and there is an element $i\in P\backslash R$ and $i\in \cap_{j\in U} S_j$. Let $S = \cap_{j\in U} S_j$ after the inner-loop. By assumption and by Lemma~\ref{lem:outside} we know $S$ has at least two elements, one of them must be $i$, and call another $i'$. By the property of $P$ we know there exists a set $S_j$ where $i\in S_j$ and $i'\not\in S_j$. Clearly $j > p$ (where $p$ is the initial element) as it contains an element outside or $R$ and $S_p$ must contain $i'$ (otherwise $i'$ will not be in $S$).

When the inner-loop goes to $j$, by Lemma~\ref{lem:spaceclose} the dimension of $Z'$ will be smaller than $Z$. Also, by the robustness we know the set at that point contains an element (namely $i$) that is not in $R$, so by Lemma~\ref{lem:outside} we know $\|P_{\Gamma^\perp}Z'\| > \alpha/2$. As a result $j$ must be added to $U$ and this contradicts with the fact that in the end $i'$ is still in $S$.
\end{proof}

Let $S = \cap_{j\in U} S_j$ after the inner-loop, finally we show in the next iteration $\Gamma$ will be $\epsilon_Y$-close to the span of vertices in $R\cup S$.

\begin{lemma}
Let $S = \cap_{j\in U} S_j$ after the inner-loop, then in the next iteration, $\Gamma$ is $\epsilon_Y$ close to row span of $W^{R\cup S}$.
\end{lemma}

\begin{proof}
Based on the hypothesis all the matrices appended to the matrix $Y$ are $\epsilon_v$-close to the span of subset of rows of $W$, and the union of all the previous subsets equal to $R$. Let $B$ be a matrix that corresponds to the matrix $Y$ with the true spans, then $\|B-Y\| \le r\epsilon_v$, and on the other hand the span of $B$ is equal to row span of $W^{R\cup S}$, with smallest nonzero singular value at least $\alpha$ (because $W$ is $\alpha$-robust). Therefore by Wedin's theorem we know since $r\epsilon_v \ll \alpha$ $\Gamma$ must be $\epsilon_Y = 2r\epsilon_v/\alpha$-close to the row span of $W^{R\cup S}$.
\end{proof}

The last two lemmas proved the second half of induction hypothesis. Finally it is easy to see that the algorithm will not stop as long as $P\backslash R$ is not empty, and it must stop after $r$ iterations because the size of $R$ increases by at least $1$ in every iteration. This concludes the proof of Theorem~\ref{thm:findintersection}

\paragraph{Finding the remaining vertices}

The proof of Theorem~\ref{thm:findremaining} follows directly from Lemma 4.5 in \cite{AroraEtAl_icml13}, for completeness we explain the proof here.
($\epsilon_v \le \alpha/20r$, $O(\epsilon/\alpha^2)$)
\begin{proof} (of Theorem~\ref{thm:findremaining})
First observe that $\alpha$-robust implies $\alpha$-robust in \cite{AroraEtAl_icml13}, because for any vertex $W^i$, let $v$ be the direction of $W^i$ projected to the orthogonal subspace of $W^{-i}$ (all the other rows). By $\alpha$-robust condition of this paper we know $\|W v\| \ge \alpha$, which in particular implies $\|P_{(W^{-i})^\perp} W^i\| \ge \alpha$.

By Lemma 4.5 in \cite{AroraEtAl_icml13}, as long as the previously found vertices are at least $\alpha/20r$-close, and all the points are $\epsilon$-close, the new vertex found by the algorithm must be $O(\epsilon/\alpha^2)$ close. Since $\epsilon/\alpha^2 \ll \epsilon_v$ we can find all the remaining vertices.

Note that we are not running the clean-up phase of Algorithm 4 FastAnchorWords, this is because the vertices we find in this phase is already more accurate than the intersection vertices and the clean-up phase cannot improve the quality of the intersection vertices (as they don't appear in $M^i$).
\end{proof}

\section{Generative model for subset-separable NMF}

In this section we prove under natural generative model an NMF problem can have $(N,H,\gamma)$-properly filled facets with high probability.

In order to prove Theorem~\ref{thm:model}, We use the following two lemmas. The first lemma shows with enough uniform points in a simplex, with high probability one of them will be a center for the simplex.

\begin{lemma}[Restating Lemma~\ref{lem:simplex}]
Given $n = \Omega((4d)^d\log d/\eta)$ uniform points $v^1,v^2,...,v^n$ in a standard $d$-dimensional simplex (with vertices $e_1,e_2,...,e_d$), with probability $1-\eta$ there exists a point $v_i$ such that $v_i = \sum_{j\ne i} w_j v^j$ ($w_j\ge 0,\sum_{j\ne i}w_j = 1$), and $\sigma_{min}(\sum_{j\ne i} w_j (v^j)(v^j)^T) \ge 1/16d$.
\end{lemma}

\begin{proof}
Consider $d+1$ subsets of the $d$-dimensional simplex: let $S_0$ be the set of points that satisfy $v_i \ge 1/2d$ for all $i\in[d]$; let $S_j$ ($j\in[d]$) be the set of points that satisfy $v_j \ge 1-1/4d$. The volume of these sets are at least $(4d)^{-d}$. By simple Chernoff bound we know when there are $n = \Omega((4d)^d\log d/\eta)$ samples, with probability at least $1-\eta$ there is a point in each of these sets.

Next we shall prove the point in $S_0$ is in the convex hull of the points in $S_j$, and the convex hull satisfies the smallest singular value requirement. First we relabel the points, let $v^0$ be any point in $S_0$ and $v^j$ be any point in $S^j$ ($j\ne 0$). Let $V\in \mathbb{R}^{d\times d}$ be the matrix whose columns are $v^j$'s ($j\in [d]$). We can apply Gershgorin's Disk Theorem to the matrix $V^T V$ (this is a matrix with diagonal entries at least $1-1/2d$ and off-diagonal entries at most $1/2d$), and conclude that $\sigma_{min}(V^T V) \ge 1/4$. 

Since in particular $V$ is full rank, let $w = V^{-1} v^0$. Let $w_i$ be the smallest entry. If $u_i < 0$ then since $\sum_{j=1}^d w_j = 1$ (all the columns of $V$ and $v^0$ sum up to 1), $\sum_{j=1}^d |u_j| \le 1-2d u_i$. The $i$-th coordinate $(Vu)_i = \sum_{j=1}^d w_j v^j_i \le (1-1/4d)w_i + (1-2dw_i)/4d \le 1/4d$, which cannot be equal to $v^0_i$, therefore $w_i \ge 0$. In this case since $1/2d \le v^0_i = (Vw)_i = \sum_{j=1}^d w_j v^j_i \le w_i + \frac{1}{4d}$, we know $w_i \ge 1/4d$. Therefore $\sigma_{min}(\sum_{j\ne i} w_j (v^j)(v^j)^T) \ge \frac{1}{4d} \sigma_{min}(VV^T) \ge \frac{1}{16d}$.
\end{proof}

Next lemma shows only subspaces that contains a properly filled facet can have many points.

\begin{lemma}[Restating Lemma~\ref{lem:subspace}]
Given $n =  \Omega(rd\log (d/p_{min}\eta)/p_{min})$ uniform points $v^1,v^2,...,v^n$ in a standard $d$-dimensional simplex (with vertices $e_1,e_2,...,e_d$), with probability $1-\eta$ for all matrices $A\in \R^{r\times d}$ whose largest column norm is equal to $1$, there are at most $p_{min}n/4$ points with $\|Av^i\| \le p_{min}/200d$.
\end{lemma}

\begin{proof}
We first prove this for a particular matrix $A$, then we will construct an $\epsilon$-net and do union bound over all possible matrices $A$.

Let $u = A_i$ where $A_i$ is the column with norm $1$. For random $v$ that is uniform in the standard $d$ dimensional simplex, we will show $\Pr[|u^T Av| \le ...] \le p_{min}/8$. By property of uniform distribution on a simplex, we know $v_i$ is independent of $v_{-i}/(1-v_i)$ (where $v_{-i}$ is the vector $v$ with $i$-th coordinate removed), and $v_i$ is distributed as a Beta distribution $Beta(1,d-1)$. Let $q = u^T Av_{-i}/(1-v_i)$, then we know  $u^TAv = v_i + (1-v_i)q$ and $q\in [-1,1]$. The density function of $v_i$ is bounded by $d-1$, therefore for any value $q$, the probability that $|u^TAv| \le p_{min}/100d$ is at most $p_{min}/8$. When the number of samples is at least $n = \Omega(\log (1/\eta')/p_{min})$, with probability $1-\eta'$ there are at most $p_{min}n/4$ points that satisfy $|u^TAv| \le p_{min}/100d$.

Now we construct an $\epsilon$-net so that for any matrix $A$ with largest column norm $1$, there is a matrix $A'$ in the $\epsilon$-net that is column-wise $\epsilon$-close to $A$. Set $\epsilon = p_{min}/200d^2$, by standard construction the number of matrices in the $\epsilon$-net is $O(d^2/p_{min})^{rd})$. Let $\eta' = \eta/O(d^2/p_{min})^{d^2})$ (and hence $n = \Omega( rd\log (d/p_{min}\eta)/p_{min})$, by union bound we know with probability $1-\eta$, there are at most $p_{min}n/4$ points with $\|Av^i\| \le p_{min}/100d$ for all matrices $A$ in the $\epsilon$-net. For a matrix $A$ that is not in the $\epsilon$-net, let $A'$ be the matrix in the net that is column-wise $\epsilon$-close, clearly $\|Av^i\| - \|A'v^i\| \le p_{min}/200d$. If there are more than $p_{min}n/4$ points with $\|Av^i\| \le p_{min}/200d$ then all these points will have $\|A'v^i\| \le p_{min}/100d$ and that is impossible.
\end{proof}

With  these two lemmas we can now prove the theorem:

\begin{proof} (of Theorem~\ref{thm:model})
In order to satisfy Condition 1, we apply Lemma~\ref{lem:simplex}. For any proper facet the points are equal to the rows of $W^{S_i}$ multiplied by uniform random points, since $\sigma_{min}(W^{S_i}) \ge \sigma_{min}(W)\ge \alpha$, we know if the facet has more than $\Omega((4d)^d\log d/\eta)$ points the convex combination has smallest singular value $\alpha^2/16d$. This is ensured when the number of samples is at least $n = \Omega((4d)^d\log (kd/\eta)/p_{min})$ by Chernoff bound.

Condition 2 is satisfied whenever $n = \Omega(\log (k/\eta)/p_{min})$ by simple Chernoff bound.

Condition 3 follows from Lemma~\ref{lem:subspace}. Suppose $Q$ is a subspace that for any proper facet $Q_i$ we have $\|P_{Q^\perp} Q_i\| > H\epsilon$. Since $W$ is $\alpha$-robust this means $\|P_{Q^\perp} W^{S_i}\| \ge H\alpha\epsilon$, therefore there is always a column that has norm $H\alpha\epsilon/\sqrt{r}$. By Lemma~\ref{lem:subspace} we know no matter which subspace the point is chosen from, with probability at most $p_{min}/8$ it will be $\epsilon/2$-close to the subspace. Now we can apply union bound to the product of all the $\epsilon$-nets constructed for different proper facets, so the size of the net is $\exp(kdr \log d)$. Therefore we know when $n = \Omega( kr^2\log (r/p_{min}\eta)/p_{min})$ (here for simplicity we used $d = r$ because in particular the interior points are in a space of dimension $r$) with high probability there will be at most $p_{min}n/4$ points for this subspace $Q$.
\end{proof}

\end{document}